\RequirePackage{iftex}\ifPDFTeX\pdfoutput=1\fi
\documentclass[runningheads,envcountsame]{llncs}

\usepackage[T1]{fontenc}
\usepackage{graphicx}
\usepackage{amsmath, amssymb}
\usepackage{mathtools}
\usepackage{booktabs}
\usepackage[bookmarks=false,hidelinks]{hyperref}

\makeatletter
\renewcommand\section{\@startsection{section}{1}{\z@}%
                       {-13\p@ \@plus -4\p@ \@minus -4\p@}%
                       {8\p@ \@plus 4\p@ \@minus 4\p@}%
                       {\normalfont\large\bfseries\boldmath
                        \rightskip=\z@ \@plus 8em\pretolerance=10000 }}
\renewcommand\subsection{\@startsection{subsection}{2}{\z@}%
                       {-13\p@ \@plus -4\p@ \@minus -4\p@}%
                       {5\p@ \@plus 4\p@ \@minus 4\p@}%
                       {\normalfont\normalsize\bfseries\boldmath
                        \rightskip=\z@ \@plus 8em\pretolerance=10000 }}
\renewcommand\paragraph{\@startsection{paragraph}{4}{\z@}%
                       {-8\p@ \@plus -4\p@ \@minus -4\p@}%
                       {-0.5em \@plus -0.22em \@minus -0.1em}%
                       {\normalfont\normalsize\itshape}}
\g@addto@macro\normalsize{%
  \setlength\abovedisplayskip{7\p@ \@plus 2\p@ \@minus 4\p@}%
  \setlength\belowdisplayskip{7\p@ \@plus 2\p@ \@minus 4\p@}%
  \setlength\abovedisplayshortskip{4\p@ \@plus 2\p@}%
  \setlength\belowdisplayshortskip{4\p@ \@plus 2\p@ \@minus 2\p@}%
}
\makeatother

\newcommand{\Rint}{r_{\mathrm{int}}}
\newcommand{\Vepi}{V_{\mathrm{epi}}}
\newcommand{\Vale}{V_{\mathrm{ale}}}
\newcommand{\Valea}{V_{\mathrm{ale}}^{\mathrm{aug}}}
\newcommand{\Vepiahead}{\widetilde V_{\mathrm{epi}}^{(h_g)}}
\newcommand{\Valeprobe}{\widetilde V_{\mathrm{ale}}^{\mathrm{probe}}}
\newcommand{\Valeraw}{\widetilde V_{\mathrm{ale}}^{\mathrm{raw}}}
\newcommand{\Etil}{\widetilde E}
\newcommand{\Ehat}{\widehat E}
\newcommand{\Ebar}{\overline E}
\newcommand{\piref}{\pi_{\mathrm{ref}}}
\newcommand{\piprobe}{\pi_{\mathrm{probe}}}
\newcommand{\Qmax}{Q_{\max}}
\newcommand{\Tref}{T^{\piref}}
\newcommand{\IEFE}{\mathrm{IEFE}}
\newcommand{\Eref}{E_{\mathrm{ref}}}
\newcommand{\eps}{\varepsilon}
\newcommand{\Acal}{\mathcal{A}}
\newcommand{\Scal}{\mathcal{S}}

\newcommand{\Var}{\mathrm{Var}}
\newcommand{\tr}{\mathrm{tr}}
\newcommand{\KL}{\mathrm{KL}}

\begin{document}

\title{Principled Direction-Free Intrinsic Motivation\\
       through Model-Free Epistemic Free-Energy Estimators}

\titlerunning{Direction-Free Intrinsic Motivation}

\author{Alireza Furutanpey\thanks{\textbf{Corresponding author}:
        \email{a.furutanpey@coovally.ai}}\fnmsep\inst{1,3}
        \and Schahram Dustdar\inst{2,3}}
\authorrunning{A. Furutanpey and S. Dustdar}

\institute{Coovally, Barcelona, Spain \and
ICREA Barcelona, Spain \and
Distributed Systems Group, TU Vienna, Austria}

\maketitle

\begin{abstract}
\looseness=-1
Across environments with mixed sources of uncertainty, unsupervised reinforcement learning requires intrinsic motivation that does not precommit to a particular direction of surprise. Surprise minimization is scoped by design to ``unstable'' environments. Prediction-error curiosity rewards total expected surprise, including irreducible noise. Bandit or mixture switching between surprise-minimizing and surprise-maximizing rewards reintroduces non-stationarity by construction. We propose a single intrinsic reward, stationary within each window, derived from the novelty contribution of a preference-free Expected Free Energy objective, expressed in reward-maximization form. Our claim is that parameter information gain, the expected surprise of the next state minus its irreducible part, is the appropriate intrinsic signal in both high-entropy and low-entropy components of the state space. Maximizing it seeks exactly the surprise the model can explain away. In regions of unresolved dynamics, this epistemic term drives exploration.  As dynamics become resolved, the epistemic term vanishes, while an aleatoric penalty favors lower-variance transitions, all without fitting an explicit next-state predictor. A pseudocount supplies epistemic value, a probe-based penalty captures aleatoric variance, and a short-horizon gate protects informative successors. A window-based freeze of all reward-defining objects yields a stationary Bellman operator, explicit bounds on learning targets, and a conditional uniform-concentration result for the nonparametric estimators under mixing, smoothness, bandwidth, and capacity assumptions. In active-inference terms, the agent is preference-free where novelty is retained, standard likelihood ambiguity vanishes under full observability, a nonstandard transition-entropy penalty is added, and surprise minimization emerges in resolved regions of the state space.
\keywords{Active inference \and Free Energy Principle \and Intrinsic Motivation \and Unsupervised Reinforcement Learning}
\end{abstract}

\section{Introduction}
\label{sec:intro}

Unsupervised reinforcement learning agents should behave appropriately across environments with mixed sources of uncertainty, spanning components where the state distribution is high-entropy under any policy and components where the agent must acquire control of low-entropy deterministic transitions. Yet the dominant intrinsic-motivation families instead precommit to a single uncertainty structure by design. Surprise minimization~\cite{berseth2021smirl} rewards negative self-information of state visitation under an episode-local density model. The construction is scoped to ``unstable'' environments and admits the dark-room failure mode~\cite{friston2012darkroom} in stable ones. Information-gain and prediction-error families~\cite{burda2019rnd,caron2025bayesexplore,pathak2017icm,pathak2019disagreement,sekar2020plan2explore} target near-deterministic or sparse-reward settings. Prediction-error variants chase stochastic distractors (the noisy-television failure). Ensemble-disagreement variants mitigate this at the cost of a learned forward-model ensemble, with members fitting the same noise converging on a shared conditional mean, so cross-member variance decays while realized prediction error remains high. Recent adaptive approaches select the surprise direction at the architecture level. S-Adapt~\cite{hugessen2024sadapt} switches between a surprise-minimizing and a surprise-maximizing reward via a per-episode UCB bandit. Mixture of Surprises~\cite{zhao2022mixture} trains parallel surprise-min and surprise-max components and switches between them on a fixed within-episode schedule. Both reintroduce non-stationarity by construction, Bellman targets shifting independently of the environment, and the two behaviors remain distinct objectives to select between.

We argue that information gain about dynamics parameters is the appropriate intrinsic signal in both high-entropy and low-entropy components of the state space, and realize it as \emph{PRIME} (Principled Direction-Free Intrinsic Motivation through Model-Free Epistemic Free-Energy Estimators), with no bandit, operator switching, or extrinsic reward. For a preference-free agent, reducing long-term surprise amounts to resolving posterior uncertainty over the dynamics of visited states. Parameter information gain is the agent's expected surprise about the next state minus the irreducible part of that surprise (Section~\ref{sec:method}), so maximizing it seeks exactly the surprise the model can explain away. Wherever dynamics are unresolved, this term dominates, and the agent explores at any entropy level. Once a region is resolved, the term vanishes, and the residual surprise is irreducible. Then, the only remaining option, enforced by the weighted aleatoric term, is to prefer actions that lead to next states with lower expected transition entropy. 
Active inference derives information seeking from long-term surprise minimization, exploration transiently raising surprise in order to lower it later~\cite{friston2015epistemic,schwartenbeck2013exploration}. PRIME derives surprise minimization from information seeking, the minimizing phase emerging from the preference-free objective once information gain is exhausted.

In active-inference terms, the agent is preference-free, with uniform preference prior $p(o\mid C)$, so the pragmatic component of expected free energy is constant in $\pi$ and drops from the argmin. The intrinsic objective weights parameter information gain against expected transition entropy, a penalty that replaces ambiguity in fully observed MDPs (Section~\ref{sec:method}). Surprise minimization in the FEP sense follows from model resolution plus aleatoric avoidance, and the principle does not require the negative self-information reward from the surprise-minimization family.

We state the target as a preference-free Expected Free Energy in reward-maximization form, keeping information gain about dynamics parameters minus the expected transition entropy of next states, and realize it model-free in the restricted sense that no next-state predictor is fit or rolled out, the statistics living on counts, scalar return estimates (return space), and conditional variances of next-state features. A count-based pseudocount~\cite{bellemare2016unifying} carries the epistemic term, a frozen ensemble of distributional return-statistics heads provides its theoretical target frame, and a short-horizon gate in the rectified novelty-difference form of~\cite{zhang2021noveld} withholds the aleatoric penalty where near-term successors retain high epistemic value. Snapshotting all reward-defining objects on a fixed schedule fixes the intrinsic reward as a function of the transition tuple in each window.

We consider our core contribution to be the argument, and its realization as PRIME, that one preference-free objective yields both surprise directions. A single window-stationary reward produces surprise-maximizing behavior where dynamics are unresolved and surprise-minimizing behavior where they are resolved, with $\gamma$-contraction, uniform target bounds, and conditional uniform concentration inside each window (Section~\ref{sec:guarantees}), in the novelty lineage of active inference~\cite{friston2015epistemic,schwartenbeck2019novelty}. The focus of this study is the theoretical foundation. The experiments (Section~\ref{sec:evaluation}) probe the predicted direction-free behavior on two didactic environments, and are readily reproducible through an openly available repository \footnote{\url{https://github.com/rezafuru/PRIME}}.
Detailed comparative analysis is deferred to an extended study. 

\section{Background and related work}
\label{sec:background}

\subsection{Free-energy principle and EFE}

The free-energy principle characterizes self-organizing systems as minimizing variational free energy, an upper bound on the negative log-evidence of observations~\cite{friston2010fep}. Active inference~\cite{dacosta2020relationship,friston2017activeinference,smith2022aitutorial} selects policies that minimize Expected Free Energy (EFE):
\begin{equation}
\begin{aligned}
G(\pi) \;=\; &-\,\mathbb E_{q(o\mid\pi)}[\log p(o\mid C)]
        \;-\; \mathbb E_{q(o\mid\pi)} \KL[q(s\mid o,\pi)\,\Vert\,q(s\mid\pi)] \\
       &-\; \mathbb E_{q(o,s\mid\pi)} \KL[q(\theta\mid o,s,\pi)\,\Vert\,q(\theta)],
\end{aligned}
\label{eq:efe}
\end{equation}
the three terms being \emph{extrinsic value}, \emph{salience} (state information gain), and \emph{novelty} (parameter information gain)~\cite{dacosta2020relationship,schwartenbeck2019novelty}. A non-negative expected-evidence-bound term, vanishing under the variational approximation $q\approx p$, is omitted. PRIME keeps the \emph{novelty} term alone, the extrinsic term constant under a uniform preference prior and the salience term dropped by design, full observability leaving it at predictive state entropy, not zero. Section~\ref{sec:method} adds a transition-entropy penalty and realizes the objective without a learned dynamics model.

\subsection{Intrinsic motivation in RL}

\looseness=-1
Intrinsic-motivation methods in RL~\cite{aubret2022survey} include prediction error~\cite{burda2019rnd,pathak2017icm}, information gain over a learned forward model~\cite{caron2025bayesexplore,houthooft2016vime,pathak2019disagreement,sekar2020plan2explore}, and density-based pseudocount bonuses~\cite{bellemare2016unifying}. Prediction-error variants suffer the noisy-television failure, treating stochastic distractors as informative. AMA~\cite{mavorparker2022ama} counters by subtracting a learned aleatoric variance, in the service of curiosity alone. Ensemble-disagreement variants mitigate the failure by rewarding cross-member predictive variance rather than realized error, requiring a forward-model ensemble. The rectified novelty-difference bracket of NovelD~\cite{zhang2021noveld} underlies the gate in Section~\ref{sec:reward}, and DEIR~\cite{wan2023deir} uses a related pattern. PTS-BE~\cite{caron2025bayesexplore} is the model-based parallel to our return-space surrogate, realizing expected parameter information gain as Jensen--Shannon disagreement across a Bayesian forward-model ensemble.

\subsection{Surprise minimization and direction-adaptive methods}

SMIRL~\cite{berseth2021smirl} maintains a density model $p_\theta(s)$ over the within-episode state buffer and uses $\log p_{\theta_{t-1}}(s_t)$ as the intrinsic reward, operationalizing the FEP reading that organisms minimize long-term surprise as negative instantaneous self-information of state visitation. The dark-room problem~\cite{friston2012darkroom} makes the cost of that operationalization explicit. In the absence of preferences, minimizing the entropy of one's own state visitation draws the agent to inactive low-entropy configurations. Friston et al.\ dissolve the problem through phenotypic priors, which a preference-free reward does not have. SMIRL avoids collapse by relying on environments whose state distributions evolve under non-agent dynamics and by limiting deployment to ``unstable'' settings.
\looseness=-1
S-Adapt~\cite{hugessen2024sadapt} and Mixture of Surprises~\cite{zhao2022mixture} retain the surprise vocabulary and select the direction architecturally (Section~\ref{sec:intro}). Section~\ref{sec:method} addresses the same question by a single window-stationary reward whose dual-direction behavior emerges from the parameter information gain net of a weighted transition-entropy penalty, without arm or mixture switching. A separate line~\cite{castanyer2024improving} achieves stationarity by augmenting the state with online-updated sufficient statistics. We instead freeze the reward-defining tuple between window boundaries, admitting the within-window $\gamma$-contraction of Section~\ref{sec:guarantees}.

\section{Setting and didactic environments}
\label{sec:setting}

\looseness=-1
We consider episodic finite-action MDPs $(\Scal,\Acal,P,\gamma)$ without extrinsic reward and with $\gamma<1$. The agent should develop control policies that perform across environment components with different conditional entropies of next state given action, under one reward weighting and no operator-side choice of surprise direction.
We use the didactic pair of~\cite{hugessen2024sadapt}, two $10\times10$ grids with $100$-step episodes matching their small variants. \textbf{Butterflies} (high-entropy): the agent catches randomly moving butterflies. The transition is non-deterministic in the non-agent component (butterfly motion), while the agent retains control over its own position. Surprise-minimizing methods suit this task, while prediction-error curiosity chases the random butterfly motion. \textbf{Maze} (low-entropy): the agent navigates to a single exit on a static grid. 
Surprise minimization collapses to any stationary inactive policy under negative self-information, and information-gain methods are appropriate.
A single agent should succeed in both environments under the \emph{same} intrinsic objective and \emph{identical} hyperparameters.

\section{Method}
\label{sec:method}
In reward-maximization form, the preference-free target is the intrinsic expected free energy
\begin{equation}
U_{\IEFE}(s,a) \;=\; I(\theta;S'\mid s,a,D) \;-\; \lambda\,\mathbb E_\theta H(S'\mid s,a,\theta).
\label{eq:iefe}
\end{equation}
The negative $-U_{\IEFE}$ is the EFE-form cost, the novelty contribution of Eq.~\eqref{eq:efe} plus the weighted transition entropy, a correspondence of sign and retained terms, not an identity, Eq.~\eqref{eq:efe} being policy-level and $U_{\IEFE}$ state-action local. The first term is parameter information gain (novelty)~\cite{schwartenbeck2019novelty}, and the second is the transition-entropy penalty, a deliberate aleatoric term. Under full observability $p(o\mid s)=\delta(o-s)$ carries no ambiguity, so the only irreducible uncertainty left to penalize is that of the dynamics $p(s'\mid s,a,\theta)$. The term is our addition, absent from the standard EFE decomposition. The two terms are linked by the mutual-information identity of BALD~\cite{houlsby2011bald}, the epistemic--aleatoric decomposition~\cite{depeweg2018decomposition},
\begin{equation}
I(\theta;S'\mid s,a,D) \;=\; H(S'\mid s,a,D) \;-\; \mathbb E_\theta H(S'\mid s,a,\theta),
\label{eq:bald}
\end{equation}
so $U_{\IEFE}(s,a) = H(S'\mid s,a,D) - (1+\lambda)\,\mathbb E_\theta H(S'\mid s,a,\theta)$, the expected surprise of the next state minus $(1+\lambda)$ times its irreducible part. Maximizing the information-gain term seeks exactly the surprise the model can explain away and vanishes once the posterior predictive matches the transition kernel, after which the weighted aleatoric term alone orders actions (Proposition~\ref{thm:neutrality}).

\subsection{Architecture}
\label{sec:arch}

A \emph{statistics ensemble} $\{E_k\}_{k=1}^K$ holds QR-DQN~\cite{dabney2018qrdqn} distributional heads, each representing the per-action return distribution $Z_k(s,a)$ by $N_q$ learned quantile locations $q_{k,i}(s,a)$ in place of a scalar mean, with scalar logit $\Ebar_k(s,a) = \tfrac{1}{N_q}\sum_i q_{k,i}(s,a)$. A \emph{control ensemble} $\{Q_k\}_{k=1}^K$ supplies action-value functions for behavior. Statistics quantiles satisfy the hard bound $|q_{k,i}|\le G_q$, and control outputs are softly regularized toward $|Q_k|\le \Qmax$ (Appendix~\ref{app:exp}). Per-head affine calibration $\Ehat_k = a_k\Ebar_k + b_k$ to a \emph{reference head} $\Eref$ with $a_k\in[a_{\min},a_{\max}]$, $|b_k|\le b_{\max}$ gives $|\Ehat_k|\le G_E := a_{\max}G_q + b_{\max}$, removing per-head gauge freedom that would contaminate cross-head variance. The ${}^{-}$ superscript denotes frozen targets. A \emph{per-state baseline} $u(\cdot\mid s)\in\Delta(\Acal)$ centers calibrated logits,
\begin{equation}
\Etil_k^{-}(s,a) \;\coloneqq\; \Ehat_k^{-}(s,a) \;-\; \sum_{b\in\Acal} u(b\mid s)\,\Ehat_k^{-}(s,b), \qquad |\Etil_k^{-}|\le 2G_E.
\end{equation}
A \emph{reference policy} $\piref(a\mid s) \propto \exp(\Ebar_{\mathrm{ref}}^{-}(s,a)/\tau)$ with $\tau\ge\tau_{\min}>0$ satisfies $\|\piref(\cdot\mid s) - u(\cdot\mid s)\|_1 \ge \delta_\pi > 0$ on a set of states of positive visitation, decoupling reward measurement from policy improvement.

\paragraph{Statistics window.} \looseness=-1 Within a window all reward-defining objects are frozen, target heads $\{E_k^{-}\}$, $\Eref^{-}$, calibration $(a_k,b_k)$, baseline $u$, $\piref$, probe directions, neighbor weights, and hyperparameters $\lambda,\alpha,\tau,h_g,\beta,\sigma_0$, all statistics under stop-gradient.

\subsection{Epistemic estimator and pseudocount realization}
\label{sec:vepi}

Under the frozen $\piref$, define the reference-policy continuation and the return-space variable
\begin{equation}
g_k(s)\;\coloneqq\;\mathbb E_{a\sim\piref(\cdot\mid s)}\!\big[\Etil_k^{-}(s,a)\big],\qquad
Y_k\;\coloneqq\;\gamma\,g_k(S').
\end{equation}
Treating the head index $k$ as uniform on $\{1,\ldots,K\}$, the variance of $Y_k$ under the joint distribution $(k,S')\mid(s,a)$ decomposes via the law of total variance as
\begin{equation}
V_{\mathrm{epi,LoTV}}(s,a) \;=\; \Var_k\!\big(\mathbb E[Y_k\mid s,a]\big), \qquad
\Vale(s,a) \;=\; \mathbb E_k\!\big(\Var_{S'\mid s,a}[Y_k]\big).
\label{eq:lotv}
\end{equation}
The decomposition is algebraic, with no Bayesian posterior semantics claimed for $\{Y_k\}$. $V_{\mathrm{epi,LoTV}}$ is the cross-head disagreement on the return-space variable, the same signal used by ensemble-disagreement curiosity~\cite{caron2025bayesexplore,pathak2019disagreement,sekar2020plan2explore}, evaluated on a value-side ensemble under shared $\piref$ policy evaluation.

\looseness=-1
The value-side ensemble, in the $K$-head construction of~\cite{osband2016bootstrapped}, exhibits documented representation-space convergence under shared TD targets~\cite{sheikh2022diversity}. Randomized priors~\cite{osband2018randomized} address prior-variance but do not by themselves decorrelate the shared TD targets. Cross-head variance therefore underestimates information gain, and Section~\ref{sec:eval-gate} measures it near $10^{-7}$. Convergence collapses only this cross-head variance, while the within-head $\Vale$ of Eq.~\eqref{eq:lotv} survives and supplies the heads' aleatoric input to the penalty and gate of Section~\ref{sec:reward}. The epistemic estimator that enters the reward is a count-based realization,
\begin{equation}
V_{\mathrm{epi,count}}(s,a) \;=\; \kappa\,\gamma^2\,\tanh^2\!\big(1/\sqrt{N_{s,a}+1}\big),
\label{eq:vepi-count}
\end{equation}
with $N_{s,a}$ a cumulative count under an augmented bucket key over the controllable state (Appendix~\ref{app:noisytv}). The realization is a bounded estimator of parameter information gain via the chain inequality $\mathrm{IG}\le\mathrm{PG}\le 1/\hat N$ of~\cite{bellemare2016unifying}, prediction gain bounding information gain and the inverse pseudocount bounding prediction gain. The shape matches this chain, decaying as $1/N_{s,a}$ for large counts while staying smooth and bounded by $\kappa\gamma^2$. $V_{\mathrm{epi,LoTV}}$ stands as the theoretical target frame for $I(\theta;S'\mid s,a,D)$, and $V_{\mathrm{epi,count}}$ as the deployed realization (Appendix~\ref{app:estimators}). The architectural contribution is the window-frozen, contraction-stable composition of Section~\ref{sec:guarantees}, and estimator tightness is not claimed.

\paragraph{Probe-augmented aleatoric detector.} With frozen feature maps $\phi,\psi$ bounded by $\|\phi(s)\|_2\le B_\phi$ and $\|\psi(s)\|_2\le B_\psi$, sample $m$ unit-sphere probes $w_i\in\mathbb S^{d-1}$ at window start and define
\begin{equation}
\Valeprobe(s,a) \;=\; \frac1m\sum_{i=1}^m \Var_{S'\mid s,a}\!\big(\gamma\,w_i^{\top}\phi(S')\big),
\end{equation}
and analogously $\Valeraw$ on $\psi$. The per-probe estimator concentrates on $\gamma^2(\tr\Sigma_\phi)/d$ under uniform sampling on $\mathbb S^{d-1}$. The aleatoric-augmented term is $\Valea(s,a) = \max\{\Vale, \Valeprobe, \Valeraw\}$, taking the maximum of the three detectors.

\paragraph{Short-horizon epistemic gate.} Let $\beta\in(0,1]$ and $h_g\in\mathbb N$. With deterministic probe policy $\piprobe(s) = \arg\max_a \piref(a\mid s)$ frozen for snapshot stability,
\begin{equation}
\Vepiahead(s,a) \;\coloneqq\; \mathbb E\!\left[\sum_{j=1}^{h_g}\beta^{j-1}V_{\mathrm{epi}}\!\big(S_{t+j},\,\piprobe(S_{t+j})\big)\,\bigg|\,S_t=s,A_t=a\right],
\end{equation}
where $V_{\mathrm{epi}} = \max(V_{\mathrm{epi,LoTV}}, V_{\mathrm{epi,count}})$ takes the larger of the two estimates.

\subsection{Intrinsic reward}
\label{sec:reward}

The intrinsic reward is, within a window,
\begin{equation}
\Rint(s,a) \;=\; V_{\mathrm{epi}}(s,a) \;-\; \lambda\,\log\!\left(1 + \frac{[\Valea(s,a) - \alpha\,\Vepiahead(s,a)]_+}{\sigma_0^2}\right).
\label{eq:reward}
\end{equation}
\looseness=-1
The first term rewards epistemic novelty. The second term penalizes the aleatoric-augmented variance net of the short-horizon epistemic look-ahead under a rectifier. The log shape saturates with bounded slope, which the bounds and the preference inequality of Section~\ref{sec:guarantees} use, and $\sigma_0^2$ sets the variance scale below which noise is ignored. In states whose near future under $\piprobe$ passes through high-$V_{\mathrm{epi}}$ states the rectifier evaluates to zero and no penalty applies. In states whose near future carries no epistemic value the penalty is preserved. At $\lambda=0$ the reward degenerates to pure epistemic novelty. The bracket $[\Valea-\alpha\Vepiahead]_+$ shares the rectified novelty-difference form of NovelD~\cite{zhang2021noveld} with positions exchanged, the aleatoric estimator in the positive term and the look-ahead epistemic value as subtractor.

\subsection{Learning rules}

Statistics heads update via QR-DQN policy evaluation under $\piref$ with target $\zeta_k^{\mathrm{tgt}} = \bar r(s,a,s') + \gamma\,q_{k,j}^{-}(s',a'_{\piref})$, $a'_{\piref}\sim\piref$, minimizing the quantile Huber loss across the $N_q$ quantile locations (distributional Bellman operator $\Pi W_1 \Tref$). Here $\bar r$ is an admissible continuation bonus, bounded and frozen for the window (Section~\ref{sec:guarantees}), in the deployed configuration a count bonus on the augmented bucket key. Control critics learn by Double-DQN~\cite{vanhasselt2016ddqn} with $y^k = \Rint(s,a) + \gamma\,Q_k^{-}(s', \arg\max_{a'} Q_k(s',a'))$ and squared-error loss. The gradient through $\Rint$ is stopped, and only the control critics update from this loss.

\section{Within-window guarantees}
\label{sec:guarantees}

\looseness=-1
We state the main claims under a window-wise freeze of all objects in Section~\ref{sec:arch}, assuming throughout $\gamma\in(0,1)$, finite $\Scal,\Acal$ (or compact and measurable), the range and calibration constraints of Section~\ref{sec:arch}, and the non-triviality margin $\delta_\pi>0$ on a set of states of positive visitation. The margin is necessary because centering gives $g_k(s)=\sum_a[\piref(a\mid s)-u(a\mid s)]\,\Ehat_k^{-}(s,a)$, so $\piref=u$ would force $g_k\equiv0$ and $V_{\mathrm{epi,LoTV}}=\Vale=0$, collapsing the return-space signal. Lipschitz conditions are stated where used.
Propositions~\ref{thm:concentration} and~\ref{thm:neutrality} rely on a multi-page nonparametric argument and an optimization hypothesis we do not establish, and are stated \emph{conditional on the cited assumptions}.
Proofs of Lemma~\ref{lem:bounds} and Proposition~\ref{thm:neutrality} and a proof sketch of Proposition~\ref{thm:concentration} are in Appendix~\ref{app:proofs}. Appendix~\ref{app:scope} maps each result to its status in the deployed tabular system.

The window freeze recovers a standard within-window problem for the otherwise non-stationary $\Rint$ (Proposition~\ref{thm:stationarity}), the contraction that the target-shifting switches of Section~\ref{sec:intro} forgo. Lemma~\ref{lem:bounds} supplies the explicit constants, and Corollary~\ref{thm:targets} turns them into a window-constant ceiling on every TD target, checked empirically in Section~\ref{sec:eval-theory}. Proposition~\ref{thm:gate} quantifies how much epistemic advantage overrides a noise difference, the noisy-television defense in inequality form. Proposition~\ref{thm:concentration} states that the estimated reward the agent acts on concentrates on the defined reward, and Proposition~\ref{thm:neutrality} formalizes the emergence claim of Section~\ref{sec:intro}, with exploration vanishing on resolved support and only the aleatoric penalty remaining.

\begin{proposition}[Stationarity and contraction]
\label{thm:stationarity}
With all reward-defining objects of Section~\ref{sec:arch} frozen, $\Rint(s,a)$ is a fixed bounded function of $(s,a)$. The Bellman optimality operator $B^{\Rint}$ acting on bounded action-value functions $Q:\Scal\times\Acal\to\mathbb R$ by $(B^{\Rint}Q)(s,a) = \Rint(s,a) + \gamma\,\mathbb E_{s'}[\max_{a'} Q(s',a')]$ is a $\gamma$-contraction on $(\ell_\infty(\Scal\times\Acal),\|\cdot\|_\infty)$, with a unique fixed point.
\end{proposition}

\begin{proof}
Freezing makes $\Rint$ deterministic on $(s,a)$ with $\|\Rint\|_\infty<\infty$ by Lemma~\ref{lem:bounds}. The $\Rint$ term cancels in the difference, so the sup-of-max bound gives $\|B^{\Rint}Q-B^{\Rint}Q'\|_\infty\le\gamma\|Q-Q'\|_\infty$. Banach's fixed-point theorem completes.
\end{proof}

\begin{lemma}[Uniform bounds]
\label{lem:bounds}
Under the calibration constraints, $|\Etil_k^{-}|\le 2G_E$, $|g_k|\le 2G_E$, and $|Y_k|\le 2\gamma G_E$. Hence, with $B_{\mathrm{ale}}\coloneqq\max\{4\gamma^2 G_E^2,\gamma^2 B_\phi^2,\gamma^2 B_\psi^2\}$,
\begin{gather}
0\le V_{\mathrm{epi,LoTV}}+\Vale\le 4\gamma^2 G_E^2, \qquad \Valea\le B_{\mathrm{ale}}, \nonumber\\
-\lambda\log\!\big(1+B_{\mathrm{ale}}/\sigma_0^2\big)\;\le\;\Rint(s,a)\;\le\; \max\{4\gamma^2 G_E^2,\,\kappa\gamma^2\}.
\end{gather}
\end{lemma}

\begin{corollary}[Bounded learning targets]
\label{thm:targets}
If additionally $|Q_k^{-}|\le\Qmax$ and $\|\bar r\|_\infty\le R_0$, then $|\zeta_k^{\mathrm{tgt}}|\le R_0+\gamma G_q$ and $|y^k|\le R_r+\gamma\Qmax$ with $R_r\coloneqq \max\{\max\{4\gamma^2 G_E^2,\kappa\gamma^2\},\,\lambda\log(1+B_{\mathrm{ale}}/\sigma_0^2)\}$.
\end{corollary}

\begin{proof}
$|\zeta_k^{\mathrm{tgt}}|\le\|\bar r\|_\infty+\gamma\,|q_{k,j}^{-}|\le R_0+\gamma G_q$. Lemma~\ref{lem:bounds} gives $\|\Rint\|_\infty\le R_r$, so $|y^k|\le\|\Rint\|_\infty+\gamma\Qmax\le R_r+\gamma\Qmax$ by the triangle inequality.
\end{proof}

\paragraph{Admissibility class.} Call a continuation bonus $\bar r$ \emph{admissible} if $\|\bar r\|_\infty\le R_0$ and $\bar r$ is frozen for the window duration. Propositions~\ref{thm:stationarity} and~\ref{thm:gate}--\ref{thm:neutrality} and Corollary~\ref{thm:targets} hold for every admissible instance with the same constants. Instances include the head-empirical decomposition of Eq.~\eqref{eq:lotv}, the count-based realization of Eq.~\eqref{eq:vepi-count}~\cite{bellemare2016unifying}, a snapshot-disciplined RND bonus with the predictor updated only at window boundaries, and a forward-model side ensemble~\cite{caron2025bayesexplore,sekar2020plan2explore} with frozen per-$(s,a)$ regression targets. Lidayan et al.~\cite{lidayan2024bamdp} prove policy-invariance and finite-horizon regret theorems for the narrower class of BAMDP potential-based shaping functions, whereas ours admits any frozen bounded $\bar r$, with within-window operator-level guarantees in place of policy-level ones.

\begin{proposition}[Quantitative gate-mediated preference]
\label{thm:gate}
For two actions $a_1,a_2$ at state $s$, let $\Delta_{\mathrm{epi}}\coloneqq V_{\mathrm{epi}}(s,a_1)-V_{\mathrm{epi}}(s,a_2)$ and $x_i\coloneqq[\Valea(s,a_i)-\alpha\Vepiahead(s,a_i)]_+/\sigma_0^2\ge 0$. WLOG $x_1\ge x_2$. A sufficient condition for $\Rint(s,a_1)\ge\Rint(s,a_2)$ is
\begin{equation}
\Delta_{\mathrm{epi}} \;\ge\; \lambda\,\frac{x_1-x_2}{1+\min\{x_1,x_2\}}.
\label{eq:gate-suff}
\end{equation}
\end{proposition}

\begin{proof}
$\Rint(s,a_1)\ge\Rint(s,a_2)$ iff $\Delta_{\mathrm{epi}}\ge\lambda[\log(1+x_1)-\log(1+x_2)]$. The mean-value theorem on $\log(1+\cdot)$ gives $\log(1+x_1)-\log(1+x_2)\le(x_1-x_2)/(1+\min\{x_1,x_2\})$ since $1/(1+x)$ is decreasing.
\end{proof}

\begin{proposition}[Uniform-in-window concentration, conditional]
\label{thm:concentration}
Assume (i) $g_k$ and the conditional-variance function $s'\mapsto\sigma_k^2(s')$ are $L$-Lipschitz in $s'$ (Lipschitzness of $g_k$ follows from Lipschitz statistics heads and $\tau\ge\tau_{\min}>0$, Section~\ref{sec:arch}; together with $|Y_k|\le 2\gamma G_E$ from Lemma~\ref{lem:bounds} this controls the variance estimator); (ii) kernel neighborhoods with bandwidth $h\to 0$ and effective neighbor count $Mh^d\to\infty$, with $M$ the number of frozen snapshots in the window and $d$ the embedding dimension; (iii) the snapshot stream is $\beta$-mixing with summable coefficients; (iv) capacity control $\log N_{\mathrm{buck}}=o(Mh^d)$, where $N_{\mathrm{buck}}$ is the covering number of $\Scal\times\Acal$ at scale $h$. With $\mu_k(s,a)\coloneqq\mathbb E[Y_k\mid s,a]$ and $\sigma_k^2(s,a)\coloneqq\Var_{S'\mid s,a}[Y_k]$ the per-bucket conditional mean and variance of $Y_k$, and $\widehat\mu_k,\widehat\sigma_k^2$ their kernel estimators, uniformly over realized buckets,
\begin{equation}
\sup_{(s,a)}|\widehat\mu_k-\mu_k|,\;\sup_{(s,a)}|\widehat\sigma_k^2-\sigma_k^2| \;=\; O_p\!\left(\sqrt{\tfrac{\log N_{\mathrm{buck}}}{Mh^d}}+h\right),
\end{equation}
and consequently $\Rint$ concentrates uniformly within the window.
\end{proposition}

\begin{proposition}[Asymptotic neutrality, conditional]
\label{thm:neutrality}
Suppose realizability and coverage hold and each statistics head converges (in the hypothesis class) to the unique fixed point of the projected distributional operator $\Pi W_1\Tref$~\cite{bellemare2017distributional,dabney2018qrdqn}. Then $V_{\mathrm{epi,LoTV}}(s,a)\to 0$ on the visited support, $V_{\mathrm{epi,count}}\to 0$ as $N_{s,a}\to\infty$, and
\begin{equation}
\Rint(s,a)\;\to\;-\lambda\,\log\!\big(1+\Valea(s,a)/\sigma_0^2\big).
\end{equation}
\end{proposition}

\paragraph{Cost of freezing.} Across windows $\Rint$ is piecewise stationary, so refresh events may shift it even when the environment is unchanged. Within a window the drift of the acting policy from $\piref$ biases $g_k$ by at most $O(\gamma G_E\|\pi_{\mathrm{act}}-\piref\|_1)$. Window length controls a bias--variance trade-off.

\section{Evaluation}
\label{sec:evaluation}

\subsection{Setup}
\label{sec:eval-setup}

\looseness=-1
The construction of Section~\ref{sec:method} is evaluated on the two didactic environments of Section~\ref{sec:setting}. The PRIME configuration is fixed at $K=5$ statistics heads, $\gamma=0.9$, gate weight $\alpha=0.5$, aleatoric weight $\lambda=0.5$, and per-seed budgets of $250{,}000$ env-steps on Butterflies and $200{,}000$ on Maze. All reward-defining hyperparameters and clip ceilings are identical across environments. Five baselines run at matched budget under their published defaults, held fixed across both environments, an analytical random walker, SMIRL~\cite{berseth2021smirl}, RND~\cite{burda2019rnd}, Disagreement~\cite{pathak2019disagreement}, and an Extrinsic-DQN reference upper bound, with $n=15$ seeds per setup. Full architectures and hyperparameters are in Appendix~\ref{app:exp}, statistical comparisons in Appendix~\ref{app:stats}. 
\subsection{Direction-free behavior}
\label{sec:eval-h1}

\begin{figure}[t]
\centering
\includegraphics[width=0.93\linewidth]{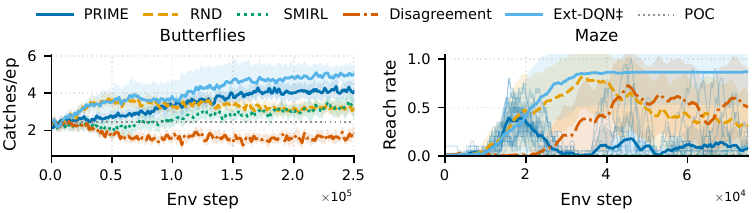}
\caption{Per-method learning curves, rolling-$20$ means of task reward, shaded $\pm1$ std ($n=15$ seeds). POC is the position-only comparator (Section~\ref{sec:eval-h1}), Ext-DQN$^{\ddagger}$ the supervised reference trained on the task reward, the Maze panel cropped at $75$k steps.}
\label{fig:f1}
\end{figure}

\begin{table}[t]
\setlength{\belowcaptionskip}{4pt}
\centering
\caption{Point estimates $\pm 1$ std ($n=15$ seeds, setup of Section~\ref{sec:eval-setup}).%
}
\label{tab:main}
\footnotesize
\setlength{\tabcolsep}{3pt}
\begin{tabular}{lccc}
\toprule
Method & Butterflies catch/ep & Butterflies peak & Maze peak \\
\midrule
Random walker        & $2.29 \pm 1.21$ & ---             & $0.00$           \\
SMIRL                & $2.71 \pm 0.21$ & $3.94 \pm 0.38$ & $0.01 \pm 0.02$  \\
Disagreement         & $1.72 \pm 0.11$ & $2.99 \pm 0.13$ & $1.00 \pm 0.00$  \\
RND                  & $3.26 \pm 0.14$ & $4.32 \pm 0.21$ & $1.00 \pm 0.00$  \\
Ext-DQN$^{\ddagger}$ & $4.11 \pm 0.67$ & $5.66 \pm 0.32$ & $0.93 \pm 0.26$  \\
\textbf{PRIME}       & $3.47 \pm 0.10$ & $5.01 \pm 0.14$ & $0.79 \pm 0.11$  \\
\bottomrule
\end{tabular}
\end{table}

\begin{figure}[t]
\centering
\includegraphics[width=0.93\linewidth]{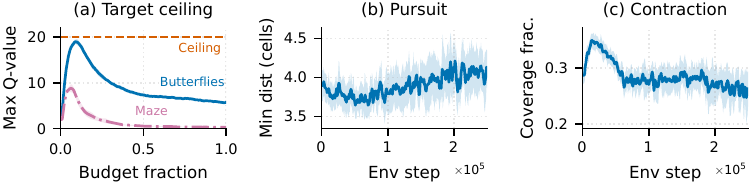}
\caption{Ceiling check and Butterflies behavior for PRIME ($n=15$ seeds, rolling-$20$ means, shaded $\pm1$ std). (a) Empirical $|q_{\mathrm{taken}}|_{\max}$ against the Corollary~\ref{thm:targets} ceiling $R_r+\gamma\Qmax$. (b) Agent--butterfly L1 min-distance. (c) Per-episode coverage fraction.}
\label{fig:f3f4}
\end{figure}

\looseness=-1
Figure~\ref{fig:f1} reports per-method learning curves, and Table~\ref{tab:main} the aggregates. On Butterflies, PRIME averages $3.47\pm0.10$ catches per episode, $1.42\times$ the position-only comparator on the same run-average metric ($2.45$, simulated for policies blind to butterfly positions), with peak rolling-$20$ $5.01\pm0.14$. On Maze, peak rolling-$20$ reach is $0.79\pm0.11$ with per-seed peaks $0.65$--$1.00$. The single-direction baselines each fail their opposite environment. SMIRL collapses on Maze ($0.01\pm0.02$) by the dark-room mechanism of Section~\ref{sec:background}, and Disagreement degrades on Butterflies below the random walker ($1.72$ vs $2.29$). RND attains both under its published defaults (Table~\ref{tab:main}, Section~\ref{sec:discussion}). PRIME attains both from one window-stationary objective under a single configuration shared across the two environments, the two directions following from the parameter-information-gain target of Section~\ref{sec:method}. Late-training reach-rate collapse on Maze is shared by every intrinsic method (PRIME last rolling-$20$ $0.04\pm0.07$), consistent with Proposition~\ref{thm:neutrality} on covered support, and the Extrinsic-DQN reference sustains at $1.00$ on $14/15$ seeds.
\label{sec:eval-theory}%
\looseness=-1
Corollary~\ref{thm:targets} bounds the control-critic targets by the window-constant $R_r+\gamma Q_{\max}$, with $R_r$ at the deployed reward clip $2.0$ (Table~\ref{tab:hp-prime}). Figure~\ref{fig:f3f4}a plots the empirical $|q_{\mathrm{taken}}|_{\max}$ trajectory against this ceiling. It tracks the soft anchor $\Qmax{=}20$ (one Butterflies seed's rolling-$20$ at $20.06$), and the bounded targets stay below the ceiling on every seed (tail-$20\%$ $|y^k|\in[6.48,7.00]$ on Butterflies, $[0.25,0.38]$ on Maze).

\subsection{Per-environment behavioral evidence}
\label{sec:eval-butter}

\looseness=-1
\textbf{Butterflies.} PRIME operates as a novelty-seeker in the $(s,a,n_{\mathrm{alive}})$ augmented bucket space. Figure~\ref{fig:f3f4}b,c plots the agent--butterfly L1 min-distance and the per-episode position-space coverage fraction. Butterfly depletion drives the dominant novelty gradient, and a coverage contraction emerges as a side effect (panel c), the fraction falling by $0.063\pm0.012$ between the first and last training quintiles (negative on $15/15$ seeds). Min-distance saturates at $4.05\pm0.09$ cells over the final $20\%$ of training (panel b), the steady-state of close pursuit on stochastically-retreating targets. The run-average catch rate above the position-only comparator (Section~\ref{sec:eval-h1}) is consistent with positional information entering the policy. A simulated butterfly end-of-episode density approaches uniform within the $100$-step horizon, so the agent's visit concentration does not reflect butterfly locations.

\looseness=-1
\textbf{Maze.} PRIME develops a near-deterministic state-conditional policy. Per-bucket $\arg\max Q$ entropy reaches $0.0015\pm0.0008$ nats at tail-$20\%$ across seeds (uniform ceiling $\log 5\approx 1.609$), and the mutual information between visited state and executed action is positive on every seed ($0.563\pm0.025$ nats, bias-corrected estimator). A random walker has zero state--action mutual information by construction, and the policy is state-conditional during the peak phase.
\subsection{Ablations}
\label{sec:eval-gate}

\begin{table}[t]
\setlength{\belowcaptionskip}{4pt}
\centering
\caption{Ablations of the full configuration. $K{=}1$ sets $V_{\mathrm{epi,LoTV}}\equiv 0$, and $\alpha{=}0$ disables the subtractor.}
\label{tab:ablations}
\footnotesize
\begin{tabular}{lccc}
\toprule
Configuration & Butterflies catch/ep & Butterflies peak & Maze peak \\
\midrule
Full ($K=5$, $\alpha=0.5$)      & $3.47 \pm 0.10$ & $5.01 \pm 0.14$ & $0.79 \pm 0.11$ \\
$K=1$                    & $3.29 \pm 0.17$ & $4.78 \pm 0.32$ & $0.70 \pm 0.11$ \\
$\alpha=0$ (gate off)    & $3.34 \pm 0.15$ & $4.93 \pm 0.18$ & $0.73 \pm 0.10$ \\
\bottomrule
\end{tabular}
\end{table}

\looseness=-1
Table~\ref{tab:ablations} reports the ablations. The full configuration leads both ablations on all six cells, the seed-paired CI of the difference excluding zero on five (Appendix~\ref{app:stats}). The $K{=}5$ gap over $K{=}1$ is $d=1.2$ on Butterflies catches, $d=1.0$ on Butterflies peak, and $d=0.9$ on Maze peak, and removing the gate costs $d=0.95$ on Butterflies catches against $d=0.6$ on Maze peak. $\Valea$ tail-$20\%$ is $2.29\times10^{-3}$ on Butterflies, the gate clamping $70\%$ of batches, against $2.03\times10^{-5}$ on Maze with $94\%$ clamped, and the cross-environment direction holds gate-off.
$V_{\mathrm{epi,LoTV}}\sim10^{-7}/10^{-8}$ confirms $V_{\mathrm{epi}}$ reduces to the count realization.

\section{Discussion}
\label{sec:discussion}

\looseness=-1
The construction realizes the parameter-epistemic component of EFE without the pragmatic component, with $V_{\mathrm{epi}}$ a surrogate for novelty~\cite{schwartenbeck2019novelty} and the rectified bracket of Eq.~\eqref{eq:reward} proxying expected transition entropy. In resolved high-entropy components, behavior is action-level aleatoric avoidance conditional on $V_{\mathrm{epi}}$ (Proposition~\ref{thm:neutrality}). The dark-room problem, dissolved in~\cite{friston2012darkroom} through phenotypic priors, does not arise without them, the agent having no incentive toward low-entropy occupancy, only smaller irreducible variability at comparable epistemic content. RND succeeds on both environments under its published defaults (Table~\ref{tab:main}), expected on near-homoskedastic environments, where prediction error aligns with expected information gain~\cite{caron2025bayesexplore}. As published, RND updates its predictor online, violating the Section~\ref{sec:guarantees} admissibility freeze. A snapshot-disciplined variant updated only at window boundaries is admissible, inherits that section's constants, and is the follow-up estimator. A mechanism-level comparison requires matched bucket keys, coverage formulas, and tuning audits, deferred to a dedicated study with a head-to-head against the switching baselines (Section~\ref{sec:background}). The non-triviality premise $\delta_\pi>0$ holds empirically, the tail-$20\%$ mean of $\|\piref-u\|_1$ at $0.0297\pm0.0007$ on Butterflies and $0.0221\pm0.0002$ on Maze, positive on every seed. The tabular count realization carries the seek phase only, the aleatoric penalty remaining on resolved support (Proposition~\ref{thm:neutrality}). Scaling to high-dimensional or continuous states needs a density-model pseudocount or representation-space disagreement signal, and the per-probe signal dilutes as $d$ grows.

\paragraph{Task-reward integration.}
\looseness=-1
PRIME trains without task reward and enters task-driven training as an exploration bonus or as reward-free pretraining adapted downstream~\cite{bellemare2016unifying,burda2019rnd,eysenbach2019diayn,laskin2021urlb}. The unsupervised-RL benchmark convention scores pretraining by efficiency of a short task-reward adaptation~\cite{laskin2021urlb}, reading the Maze trajectory of Section~\ref{sec:eval-h1} at its peak. The peak checkpoint is the transferable object (reach $0.79\pm0.11$), the decay the predicted preference-free behavior (Proposition~\ref{thm:neutrality}), and vanishing epistemic reward signals a resolved region. As an additive bonus $r_{\mathrm{ext}}+\eta\,\Rint$, the log-penalty is bounded (Lemma~\ref{lem:bounds}), ordering actions of equal epistemic content without overriding a task reward above that bound (Proposition~\ref{thm:gate}), and $\eta$ can anneal as the epistemic signal vanishes without distractor-seeking residue (Eq.~\eqref{eq:bald}). In reward-sparse settings the epistemic term densifies the signal until task reward takes over. Aleatoric avoidance minimizes outcome variance, aligning with risk-averse control, conflicting with risk-neutral return maximization on high-variance task states. A bonus-weight scheduler~\cite{riedmiller2018sacx} is the concrete next step.

\bibliographystyle{splncs04}
\bibliography{references}

\appendix

\section{Experimental details}
\label{app:exp}

\paragraph{Environments.} Both are $10\times10$ grids with action set $\{\text{N},\text{S},\text{E},\text{W},\text{Stay}\}$, $100$-step episodes, fixed-horizon truncation (no termination), and observation tensors in $\{0,1\}^{C\times10\times10}$. \textbf{Butterflies} ($C=4$ channels: agent, butterfly, wall, caught-flash) places $6$ butterflies at uniform-random empty cells. After each agent move every alive butterfly samples a move uniformly from the five actions with wall/bound blocking, and a butterfly sharing the agent cell is caught and removed (no respawn). The per-step butterfly motion is the aleatoric component, and depletion within the $100$-step horizon removes it. \textbf{Maze} ($C=3$ channels: agent, wall, flag) is a fixed S-corridor with a consumable goal flag at $(9,9)$ and deterministic transitions. Both report the task signal (catches, goal reach) outside the observation and emit zero per-step reward to the intrinsic agent. Per-seed budget is $250{,}000$ env-steps on Butterflies and $200{,}000$ on Maze.

\paragraph{Shared trunk.} Every network uses the same two-layer convolutional trunk, two $3\times3$ convolutions taking $C\to16\to32$ channels with padding $1$ and a ReLU after each, flattened to $3200$ and linearly mapped to a $64$-dim embedding.

\paragraph{PRIME networks.} The statistics ensemble holds $K=5$ QR-DQN heads, each a trunk plus a quantile output layer $\mathrm{Linear}(64,|\Acal|N_q)$ with $N_q=11$ learned quantile locations passed through $\tanh\cdot G_q$ ($|q_{k,i}|\le G_q=2$), wrapped with a frozen randomized prior of strength $\beta_{\mathrm{prior}}=2$ and reseeded per head, with targets hard-synced at window boundaries. The control ensemble holds $K=5$ unbounded $Q$-heads $\mathrm{Linear}(64,|\Acal|)$ trained by Double-DQN, with the range $|Q_k|\le\Qmax=20$ enforced softly by a penalty $0.1\cdot\mathrm{ReLU}(|Q|-\Qmax)^2$. The aleatoric features are an ICM inverse-dynamics encoder $\phi$ (trunk to $\mathrm{Linear}(3200,32)\,\tanh$, $\dim 32$, with inverse head $\mathrm{Linear}(64,64)\,\mathrm{ReLU}\,\mathrm{Linear}(64,|\Acal|)$) and a fixed $\ell_2$-normalized flattened observation $\psi$ ($\dim C{\cdot}100$). The count realization keys on $(\text{agent row},\text{agent col},a,n_{\mathrm{alive}})$.

\paragraph{Hyperparameters.} Table~\ref{tab:hp-prime} lists the PRIME configuration, byte-identical across the two environments except for the environment handle. Table~\ref{tab:hp-base} lists the baseline configuration, shared across all four DQN baselines and both environments.

\begin{table}[h]
\footnotesize
\setlength{\tabcolsep}{5pt}
\centering
\caption{PRIME hyperparameters (identical across both environments).}
\label{tab:hp-prime}
\begin{tabular}{llll}
\toprule
$K$ & $5$ & discount $\gamma$ & $0.9$ \\
quantile locations $N_q$ & $11$ & aleatoric weight $\lambda$ & $0.5$ \\
quantile bound $G_q$ & $2.0$ & gate weight $\alpha$ & $0.5$ \\
control bound $\Qmax$ & $20.0$ & log scale $\sigma_0^2$ & $0.5$ \\
calibration $a_{\min}/a_{\max}/b_{\max}$ & $1.0/2.0/1.0$ & gate decay $\beta$ & $0.95$ \\
randomized-prior $\beta_{\mathrm{prior}}$ & $2.0$ & gate horizon $h_g$ & $4$ \\
count scale $\kappa$ & $0.5$ & temperature $\tau$ & $1.0$ \\
probes $m$ & $8$ & feature dim $\phi$ & $32$ \\
window length & $1000$ & warmup window & $1000$ \\
snapshot buffer & $5000$ & replay buffer & $100000$ \\
calibration samples & $256$ & neighbors & $16$ \\
batch size & $64$ & target sync & $1000$ \\
update period (main/stats/inv) & $4/4/4$ & $\eps_{\mathrm{start}}/\eps_{\mathrm{end}}$ & $1.0/0.01$ \\
$\mathrm{lr}$ control/stats/$\phi$ & $3{\cdot}10^{-4}/10^{-4}/10^{-3}$ & reward scale/clip & $50/2.0$ \\
\bottomrule
\end{tabular}
\end{table}

\begin{table}[h]
\footnotesize
\setlength{\tabcolsep}{6pt}
\centering
\caption{Baseline hyperparameters (shared across all DQN baselines and both environments).}
\label{tab:hp-base}
\begin{tabular}{llll}
\toprule
discount $\gamma$ & $0.99$ & batch size & $64$ \\
learning rate & $3{\cdot}10^{-4}$ & target sync & $1000$ \\
replay buffer & $100000$ & learning starts & $1000$ \\
$\eps_{\mathrm{end}}$ & $0.05$ & exploration fraction & $0.25$ \\
gradient clip & $10.0$ & optimizer & Adam \\
\bottomrule
\end{tabular}
\end{table}

All networks use Adam at its default moments $(0.9,0.999)$ and $\eps=10^{-8}$. Disagreement uses a $K=5$ forward-model ensemble with prediction clip $5.0$; RND uses a frozen random target and a predictor trained on the same replay minibatches as its DQN update~\cite{burda2019rnd}, with reward clip $5.0$; SMIRL augments the observation with a per-cell density map and a time plane. The three intrinsic baselines normalize the intrinsic reward by a running standard deviation. Each configuration runs $n=15$ seeds, seeds $5$--$14$ collected after seeds $0$--$4$ under the identical configuration and code. The exploration schedule is linear over the first $25\%$ of training. Control critics learn from $y^k=\Rint(s,a)+\gamma Q_k^{-}(s',\arg\max_{a'}Q_k(s',a'))$ with the gradient through $\Rint$ stopped.

\paragraph{Metrics.} Catches per episode is the per-seed mean of the episode catch count over all episodes of a run, reported as mean $\pm$ std across seeds. Peak is the per-seed maximum over training of the rolling mean across $20$ consecutive episodes, on catches for Butterflies and on goal reach for Maze, and last rolling-$20$ is the final value of the same window. An episode counts as a reach when the goal flag is collected at least once. Tail-$20\%$ denotes the mean of a logged scalar over the final $20\%$ of env-steps. The coverage fraction is the per-episode fraction of the $100$ cells visited, and the pursuit distance is the per-episode mean of the per-step L1 distance to the nearest alive butterfly. The random-walker reference is simulated for $10^4$ episodes on the deployed environment, and the position-only comparator ($2.45$ catches/ep) is the highest catch rate among the simulated policies blind to butterfly positions, attained by a go-to-center-and-stay policy. The state--action mutual information is estimated every $1000$ steps from a ring buffer of executed actions keyed by the source-state components of the count bucket (agent position and $n_{\mathrm{alive}}$), with Miller--Madow correction on the marginal and every per-bucket conditional entropy, buckets under five samples dropped, and the estimate clamped at zero. The per-bucket $\arg\max Q$ entropy is the bucket-size-weighted mean, over snapshot buckets with at least four samples, of the entropy of the $\arg\max Q$ action across the bucket's snapshot states.

\paragraph{Hyperparameter provenance.} The baselines run at their published settings, held fixed across both environments (Table~\ref{tab:hp-base}). None was tuned per environment or per task, and no baseline sweep was performed. PRIME's single configuration (Table~\ref{tab:hp-prime}) was selected once for both environments jointly and frozen. The only configured differences between PRIME and the baselines are the discount ($\gamma=0.9$ vs.\ $0.99$) and the exploration floor ($\eps_{\mathrm{end}}=0.01$ vs.\ $0.05$), alongside the intrinsic-reward components absent from the baselines. The comparison therefore grants the baselines their published settings at matched budget while constraining PRIME to one configuration, and gives PRIME no per-environment tuning advantage.

\section{Statistical analysis}
\label{app:stats}

Each configuration runs $n=15$ seeds. Interval estimation is primary~\cite{agarwal2021precipice}, namely the per-method IQM with a stratified-bootstrap $95\%$ CI (Table~\ref{tab:iqm}), the bootstrap $95\%$ CI of the seed-mean difference, Cohen's $d$, and the probability of improvement $P(A{>}B)$ (all run pairs, ties half-weighted), with rank tests confirmatory. Seeds index matched blocks, the seed fixing the server assignment shared by every method, so the confirmatory test is the exact two-sided Wilcoxon signed-rank on per-seed differences, with the exact Mann--Whitney $U$ alongside, each Holm-corrected across the eleven-comparison family. At $n=15$ the exact signed-rank floor is $2/2^{15}\approx6.1\times10^{-5}$, against a best rank-test floor of $2/\binom{10}{5}=0.0079$ (exact Mann--Whitney) at the previous $n=5$, where no Holm-corrected rank test could reach $\alpha=0.05$. All intervals use $10^4$ percentile resamples under a fixed stream, resampling seeds within each configuration for the method CIs and per-seed differences for the difference CIs. Table~\ref{tab:stats} reports the comparisons against Table~\ref{tab:main} and the ablations, with per-seed values in Table~\ref{tab:seed}.

\begin{table}[h]
\footnotesize
\centering
\caption{Pairwise comparisons ($n=15$). $\Delta$ is the seed-mean difference (first minus second), CI its seed-paired bootstrap $95\%$ interval, $d$ Cohen's $d$, PoI the probability of improvement, and $p_{\mathrm{W}}$/$p_{\mathrm{U}}$ the Holm-corrected exact Wilcoxon signed-rank and Mann--Whitney $p$. Bold marks $p<0.05$. BF, Butterflies; c, catch/ep; p, peak.}
\label{tab:stats}
\setlength{\tabcolsep}{2.5pt}
\begin{tabular}{lcccccc}
\toprule
Comparison (metric) & $\Delta$ & $95\%$ CI & $d$ & PoI & $p_{\mathrm{W}}$ & $p_{\mathrm{U}}$ \\
\midrule
PRIME $-$ RND (BF c)            & $+0.210$ & $[+0.137,+0.284]$ & $1.67$ & $0.89$ & $\mathbf{0.002}$ & $\mathbf{<10^{-3}}$ \\
PRIME $-$ RND (BF p)            & $+0.693$ & $[+0.580,+0.793]$ & $3.94$ & $1.00$ & $\mathbf{<10^{-3}}$ & $\mathbf{<10^{-3}}$ \\
PRIME $-$ SMIRL (Maze p)        & $+0.783$ & $[+0.733,+0.837]$ & $10.0$ & $1.00$ & $\mathbf{<10^{-3}}$ & $\mathbf{<10^{-3}}$ \\
PRIME $-$ Disag.\ (BF c)        & $+1.743$ & $[+1.660,+1.825]$ & $16.1$ & $1.00$ & $\mathbf{<10^{-3}}$ & $\mathbf{<10^{-3}}$ \\
PRIME $-$ Disag.\ (BF p)        & $+2.023$ & $[+1.903,+2.127]$ & $14.7$ & $1.00$ & $\mathbf{<10^{-3}}$ & $\mathbf{<10^{-3}}$ \\
\midrule
Full $-$      $K{=}1$ (BF c)    & $+0.176$ & $[+0.065,+0.299]$ & $1.23$ & $0.81$ & $0.075$ & $\mathbf{0.019}$ \\
Full $-$      $K{=}1$ (BF p)    & $+0.233$ & $[+0.047,+0.427]$ & $0.96$ & $0.74$ & $0.192$ & $0.147$ \\
Full $-$      $K{=}1$ (Maze p)  & $+0.093$ & $[+0.020,+0.170]$ & $0.87$ & $0.70$ & $0.231$ & $0.244$ \\
Full $-$      gate-off (BF c)   & $+0.122$ & $[+0.027,+0.206]$ & $0.95$ & $0.73$ & $0.128$ & $0.147$ \\
Full $-$      gate-off (BF p)   & $+0.080$ & $[-0.050,+0.197]$ & $0.50$ & $0.64$ & $0.231$ & $0.297$ \\
Full $-$      gate-off (Maze p) & $+0.067$ & $[+0.003,+0.133]$ & $0.63$ & $0.67$ & $0.231$ & $0.297$ \\
\bottomrule
\end{tabular}
\end{table}

All five PRIME--baseline comparisons clear $\alpha=0.05$ under the Holm-corrected exact signed-rank, with probability of improvement $0.89$--$1.00$ and paired CIs excluding zero. PRIME separates from RND on both Butterflies metrics at $n=15$ (catch $d=1.67$, peak $d=3.94$), where the $n=5$ sample left the catch comparison unresolved. The full configuration leads the ablations on every cell with paired CIs excluding zero on five of six. The $K{=}1$ Butterflies-catch effect clears the Holm family on the Mann--Whitney test ($p=0.019$, signed-rank $p=0.075$), and the remaining ablation effects are medium, $d=0.5$--$1.0$, without surviving the correction. Holm-corrected Welch's $t$, computed as a parametric check, agrees with the Mann--Whitney column on all eleven comparisons. The Extrinsic-DQN Maze mean reflects seed $7$, which never reaches the goal within budget under the sparse extrinsic reward, while its IQM stays $1.00$ (Table~\ref{tab:iqm}).

\begin{table}[h]
\footnotesize
\setlength{\tabcolsep}{4pt}
\centering
\caption{Per-method IQM with stratified-bootstrap $95\%$ CI ($n=15$).}
\label{tab:iqm}
\begin{tabular}{lccc}
\toprule
Method & Butterflies catch/ep & Butterflies peak & Maze peak \\
\midrule
PRIME      & $3.47$ $[3.42,3.53]$ & $5.03$ $[4.94,5.09]$ & $0.78$ $[0.72,0.86]$ \\
$K{=}1$    & $3.32$ $[3.21,3.40]$ & $4.82$ $[4.63,4.97]$ & $0.72$ $[0.64,0.77]$ \\
gate-off   & $3.36$ $[3.25,3.44]$ & $4.94$ $[4.82,5.04]$ & $0.72$ $[0.67,0.78]$ \\
RND        & $3.25$ $[3.17,3.33]$ & $4.30$ $[4.19,4.43]$ & $1.00$ $[1.00,1.00]$ \\
SMIRL      & $2.70$ $[2.61,2.81]$ & $3.94$ $[3.69,4.16]$ & $0.00$ $[0.00,0.02]$ \\
Disag.     & $1.71$ $[1.67,1.78]$ & $2.99$ $[2.93,3.06]$ & $1.00$ $[1.00,1.00]$ \\
Ext-DQN    & $4.16$ $[3.70,4.51]$ & $5.72$ $[5.49,5.86]$ & $1.00$ $[1.00,1.00]$ \\
\bottomrule
\end{tabular}
\end{table}

\begin{table}[h]
\scriptsize
\setlength{\tabcolsep}{2.5pt}
\centering
\caption{Per-seed values (seeds $0$--$14$). BF catch/ep and peak from the $250$k sweep, Maze peak from the $200$k sweep.}
\label{tab:seed}
\begin{tabular}{lccccccccccccccc}
\toprule
Method & 0&1&2&3&4&5&6&7&8&9&10&11&12&13&14 \\
\midrule
& \multicolumn{15}{c}{BF catch/ep} \\
\cmidrule(lr){2-16}
PRIME     & 3.44&3.21&3.61&3.59&3.43&3.44&3.58&3.52&3.49&3.42&3.55&3.39&3.40&3.38&3.53 \\
$K{=}1$   & 3.43&3.39&3.23&3.39&3.22&3.50&2.93&3.28&3.23&3.21&2.94&3.32&3.46&3.45&3.37 \\
gate-off  & 3.26&3.50&3.34&3.46&3.09&3.12&3.29&3.51&3.25&3.31&3.40&3.49&3.55&3.15&3.43 \\
RND       & 3.41&3.23&3.18&3.38&3.39&3.02&3.21&3.20&3.30&3.22&3.56&3.11&3.19&3.08&3.35 \\
SMIRL     & 3.01&2.75&2.55&2.86&2.24&3.09&2.72&2.69&2.65&2.57&2.86&2.54&2.76&2.74&2.58 \\
Disag.    & 1.84&1.78&1.65&1.65&1.48&1.68&1.78&1.65&1.66&1.82&1.66&1.74&1.96&1.69&1.80 \\
Ext-DQN   & 4.89&4.74&5.12&4.44&3.28&4.30&3.13&4.52&3.06&4.26&4.43&3.22&4.39&4.11&3.74 \\
\midrule
& \multicolumn{15}{c}{BF peak} \\
\cmidrule(lr){2-16}
PRIME     & 4.85&4.70&5.00&5.15&5.00&4.85&5.15&5.05&4.95&4.95&5.20&5.05&5.10&5.00&5.20 \\
$K{=}1$   & 5.20&4.95&4.90&5.05&4.60&5.00&4.50&4.80&4.45&4.85&4.05&4.45&4.75&5.15&5.00 \\
gate-off  & 5.00&5.25&4.95&4.95&4.70&4.70&4.90&5.05&4.65&5.10&4.70&5.10&5.10&4.90&4.95 \\
RND       & 4.55&4.10&4.30&4.50&4.50&4.05&4.30&4.10&4.75&4.25&4.55&4.15&4.20&4.20&4.30 \\
SMIRL     & 4.05&3.65&4.00&4.20&3.45&4.45&3.90&3.75&3.45&3.65&4.25&3.40&4.05&4.65&4.25 \\
Disag.    & 2.85&3.25&3.05&3.05&2.70&3.15&3.00&2.90&2.95&3.00&2.95&3.00&3.10&3.05&2.85 \\
Ext-DQN   & 6.00&5.75&6.00&5.95&5.25&5.85&5.10&5.90&5.40&5.85&5.60&5.05&5.80&5.65&5.70 \\
\midrule
& \multicolumn{15}{c}{Maze peak} \\
\cmidrule(lr){2-16}
PRIME     & .95&.85&.75&.90&.65&.80&.70&.75&.80&.70&.90&.70&.65&1.0&.80 \\
$K{=}1$   & .65&.50&.80&.70&.70&.70&.80&.50&.85&.80&.75&.65&.60&.70&.80 \\
gate-off  & .70&.80&.60&.55&.75&.80&.65&.75&.65&.65&.70&.70&.80&.95&.85 \\
RND       & 1.0&1.0&1.0&1.0&1.0&1.0&1.0&1.0&1.0&1.0&1.0&1.0&1.0&1.0&1.0 \\
SMIRL     & .05&.00&.00&.00&.00&.05&.05&.00&.00&.00&.00&.00&.00&.00&.00 \\
Disag.    & 1.0&1.0&1.0&1.0&1.0&1.0&1.0&1.0&1.0&1.0&1.0&1.0&1.0&1.0&1.0 \\
Ext-DQN   & 1.0&1.0&1.0&1.0&1.0&1.0&1.0&.00&1.0&1.0&1.0&1.0&1.0&1.0&1.0 \\
\bottomrule
\end{tabular}
\end{table}

Table~\ref{tab:mi} reports the per-seed state--action mutual information of the executed policy (estimator in Appendix~\ref{app:exp}), the Maze value underlying Section~\ref{sec:eval-butter}. The measure is positive on every seed in both environments and sits below the uniform ceiling $\log 5\approx1.609$ nats, while a state-independent policy is zero by construction. Butterflies carries the larger value ($1.007\pm0.018$ against $0.563\pm0.025$ nats). The Butterflies estimator keys on agent position and $n_{\mathrm{alive}}$ (Appendix~\ref{app:exp}), so its value reflects conditioning on catch progress alongside position, while the Maze value, keyed on position alone, is the unconfounded evidence of a state-conditional policy.

\begin{table}[h]
\scriptsize
\centering
\caption{Per-seed state--action mutual information for PRIME (tail-$20\%$ means, nats). Butterflies from the $250$k sweep, Maze from the $200$k sweep.}
\label{tab:mi}
\setlength{\tabcolsep}{1pt}
\begin{tabular}{lccccccccccccccc}
\toprule
Environment & 0&1&2&3&4&5&6&7&8&9&10&11&12&13&14 \\
\midrule
Butterflies & 1.043&0.989&1.018&1.016&1.012&1.001&1.022&1.002&1.016&0.985&1.029&0.991&0.989&1.012&0.978 \\
Maze        & 0.529&0.560&0.550&0.616&0.535&0.606&0.583&0.568&0.567&0.568&0.577&0.551&0.540&0.537&0.557 \\
\bottomrule
\end{tabular}
\end{table}

\section{Target frame and count realization}
\label{app:estimators}

The novelty term $I(\theta;S'\mid s,a,D)$ admits more than one model-free estimator within the admissible class of Section~\ref{sec:guarantees}. The return-space cross-head variance $V_{\mathrm{epi,LoTV}}$ of Eq.~\eqref{eq:lotv} is the theoretical-target frame, the ensemble-disagreement signal on a value-side ensemble under shared $\piref$ policy evaluation, with no posterior semantics claimed. The count $V_{\mathrm{epi,count}}$ of Eq.~\eqref{eq:vepi-count} is the operative estimator in the reward. Both are admissible bounded continuation bonuses, so Proposition~\ref{thm:stationarity}, Lemma~\ref{lem:bounds}, Corollary~\ref{thm:targets}, and Propositions~\ref{thm:gate}--\ref{thm:neutrality} hold for either with the same constants.

Both bound the same quantity. The count satisfies the chain $\mathrm{IG}\le\mathrm{PG}\le 1/\hat N$ of~\cite{bellemare2016unifying}, so $V_{\mathrm{epi,count}}=\kappa\gamma^2\tanh^2(1/\sqrt{N_{s,a}+1})$ is a bounded estimator of parameter information gain decaying as $1/N_{s,a}$. On the finite didactic state spaces the count is exact, an admissible realization of the novelty term, and the return-space variance of Eq.~\eqref{eq:lotv} is the alternative realization in the same class.

The return-space frame collapses on these environments. Under shared $\piref$ policy evaluation the $K$ heads converge to a common fixed point of $\Pi W_1\Tref$, a documented property of value-side ensembles under shared TD targets~\cite{sheikh2022diversity}. Randomized priors~\cite{osband2018randomized} address prior variance but do not decorrelate the targets. Cross-head variance therefore underestimates information gain, measured near $10^{-7}$ (Section~\ref{sec:eval-gate}). The same convergence drives Proposition~\ref{thm:neutrality}. The alignment of the conditional means $\mathbb E[Y_k\mid s,a]$ that sends $V_{\mathrm{epi,LoTV}}\to0$ is the mechanism by which the reward reduces to the aleatoric penalty on resolved support. Persistent cross-head disagreement would require decorrelated targets or a parameter-side ensemble outside the shared-$\piref$ evaluation, a different construction left to a scalable extension where the count is infeasible.

\begin{figure}[t]
\centering
\includegraphics[width=\linewidth]{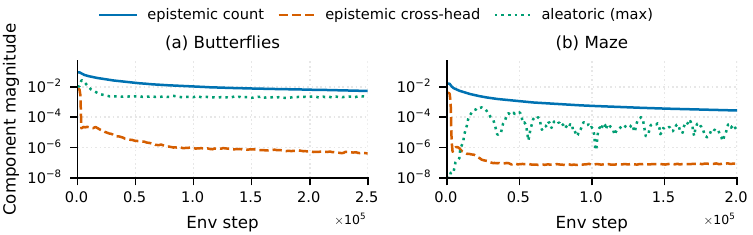}
\caption{Reward-component magnitudes for PRIME, seed-mean traces ($n=15$) of the count realization, the cross-head variance, and the aleatoric term $\Valea$, rolling means over logged batches on a log scale with display floor $10^{-8}$. Butterflies from the $250$k sweep, Maze from the $200$k sweep.}
\label{fig:components}
\end{figure}

Figure~\ref{fig:components} traces the deployed reward components over training. On Butterflies the cross-head variance falls from a $5.9\times10^{-4}$ first-decile mean to a $5.0\times10^{-7}$ tail-$20\%$ mean while the count decays from $5.0\times10^{-2}$ to $6.0\times10^{-3}$, and on Maze the corresponding tails are $8.6\times10^{-8}$ against $3.1\times10^{-4}$, the count ending three to four orders of magnitude above the cross-head variance on both environments. The count decays under accumulating coverage toward the Proposition~\ref{thm:neutrality} limit. The aleatoric term separates the environments, $2.3\times10^{-3}$ on stochastic Butterflies against $2.0\times10^{-5}$ on deterministic Maze.

\section{Noisy-TV avoidance and the controllable bucket key}
\label{app:noisytv}

The epistemic count keys on $(\text{agent row},\text{agent col},a,n_{\mathrm{alive}})$, the controllable abstraction, excluding butterfly positions. Butterfly motion is a parameter-free random walk, so by the decomposition of Eq.~\eqref{eq:bald} its parameter information gain is zero and a count estimating $I(\theta;S')$ should not count it. Computing novelty over a controllable or task-relevant abstraction is standard. Contingency-aware exploration counts agent-controllable position and excludes uncontrollable moving objects~\cite{choi2019contingency}; feature-space pseudocounts count feature-relevant structure in place of raw stochastic configurations~\cite{martin2017count}; inverse-dynamics curiosity~\cite{pathak2017icm} and representation-change rewards~\cite{raileanu2020ride} exclude uncontrollable components by construction.

The aleatoric detectors read the full observation. The probe $\Valeraw$ operates on $\psi$, the $\ell_2$-normalized full observation including the butterfly channel, and its conditional variance over the $(s,a)$ bucket registers the butterfly motion. The policy receives butterfly positions as input, and the run-average catch rate above the position-only comparator (Section~\ref{sec:eval-h1}) is consistent with their use. Only the epistemic count excludes them, by the epistemic--aleatoric assignment.

The didactic environments do not exercise active noise avoidance. Butterflies are captured without replacement, so the stochastic component depletes within the $100$-step horizon and no persistent distractor remains. RND, on the full observation, attains both environments (Table~\ref{tab:main}). The aleatoric penalty is correspondingly small on Butterflies ($\Valea$ tail-$20\%$ $2.3\times10^{-3}$, Section~\ref{sec:eval-gate}) and the cross-environment direction holds with the gate disabled (Table~\ref{tab:ablations}). The experiments demonstrate direction-free epistemic behavior under one configuration. Active aleatoric avoidance is a property of the objective (Eq.~\eqref{eq:bald}, Proposition~\ref{thm:gate}) whose demonstration needs a persistent-noise environment and is deferred.

\section{Theory and the deep-RL implementation}
\label{app:scope}

The within-window results split by what the deployed system requires (Table~\ref{tab:scope}). The deployed epistemic term is the tabular count, for which Proposition~\ref{thm:stationarity} and the bounds of Lemma~\ref{lem:bounds} hold exactly, and the target ceiling of Corollary~\ref{thm:targets} requires none of the kernel, bandwidth, or $\beta$-mixing assumptions. The ceiling is checked against the empirical $|q_{\mathrm{taken}}|_{\max}$ in Figure~\ref{fig:f3f4}a. The non-parametric concentration of Proposition~\ref{thm:concentration} describes the scalable estimator class, and asymptotic neutrality (Proposition~\ref{thm:neutrality}) is conditional on each distributional head converging in-hypothesis-class to the projected fixed point.

\begin{table}[h]
\footnotesize
\setlength{\tabcolsep}{6pt}
\centering
\caption{Each result against its status in the deployed tabular-count system.}
\label{tab:scope}
\begin{tabular}{lll}
\toprule
Result & Status in deployment & Function approximation \\
\midrule
Prop.~\ref{thm:stationarity} contraction & exact & none \\
Lemma~\ref{lem:bounds} bounds & exact & none \\
Cor.~\ref{thm:targets} target ceiling & conditional (Fig.~\ref{fig:f3f4}a) & soft $|Q_k|\le\Qmax$ \\
Prop.~\ref{thm:gate} gate preference & exact (specified reward) & none \\
Prop.~\ref{thm:concentration} concentration & conditional, scalable class & yes \\
Prop.~\ref{thm:neutrality} neutrality & conditional (in-class) & yes (statistics heads) \\
\bottomrule
\end{tabular}
\end{table}

Three premises hold only approximately in the deep implementation. A learned encoder is a fixed-kernel machine only under lazy training~\cite{jacot2018ntk}. Once the features themselves move, the effective kernel and covering number drift, and bootstrapping can collapse the feature geometry~\cite{lyle2022capacity}, so the fixed-bandwidth and capacity premises of Proposition~\ref{thm:concentration} are approximate. The snapshot stream is generated by a continually updated network under a non-stationary acting policy, so the stationary $\beta$-mixing premise reads as its non-stationary extension. In-class convergence is an assumption we do not establish. The distributional optimality operator is not a contraction and need not admit a fixed point~\cite{bellemare2017distributional}, and the quantile contraction is proved for fixed-policy evaluation~\cite{dabney2018qrdqn}. These errors fall on the return-space term, already collapsed, and on the control critics, whose target ceiling is verified directly (Fig.~\ref{fig:f3f4}a) under a softly enforced $|Q_k|\le\Qmax$. The deadly triad of function approximation, bootstrapping, and off-policy updates~\cite{vanhasselt2018triad} bears on the function-approximation components, and the mitigations adopted (target networks, Double-DQN, the window-capped bootstrap magnitude) are those it identifies as reducing divergence. The tabular within-window guarantees are untouched.

\section{Deferred proofs}
\label{app:proofs}

\begin{proof}[Proof of Lemma~\ref{lem:bounds}]
Probability-measure centering at most doubles the $G_E$ bound on $\Ehat_k^{-}$. Averaging under $\piref$ and multiplying by $\gamma$ yield $|g_k|\le 2G_E$ and $|Y_k|\le 2\gamma G_E$. Popoviciu's inequality on the $4\gamma G_E$ range gives $\Var(Y_k)\le 4\gamma^2 G_E^2$, and the law of total variance distributes this between $V_{\mathrm{epi,LoTV}}$ and $\Vale$. $V_{\mathrm{epi,count}}\le\kappa\gamma^2$ since $\tanh^2\le 1$. The probe bounds follow from $|w_i^\top\phi(S')|\le B_\phi$ and Popoviciu. The reward bound follows from monotonicity of $\log(1+\cdot)$.
\end{proof}

\begin{proof}[Proof sketch of Proposition~\ref{thm:concentration}]
The rate is standard for $\beta$-mixing nonparametric regression under capacity control. The argument uses U-statistic decoupling, blocking, and covering-number bounds~\cite{yu1994mixing} and is conditional on (i)--(iv).
\end{proof}

\begin{proof}[Proof of Proposition~\ref{thm:neutrality}]
Shared-target convergence aligns the conditional means $\mathbb E[Y_k\mid s,a]$ across $k$, driving $V_{\mathrm{epi,LoTV}}\to 0$ (conditional on in-class convergence). The pseudocount realization satisfies
\[ V_{\mathrm{epi,count}}=\kappa\gamma^2\tanh^2\!\big(1/\sqrt{N_{s,a}+1}\big)\to 0 \quad\text{as } N_{s,a}\to\infty \]
by direct calculation. The aggregated $\Vepi\to0$ on the visited support gives $\Vepiahead\to0$, and substituting in Eq.~\eqref{eq:reward} yields the stated limit.
\end{proof}

\end{document}